\documentclass[letterpaper, 10pt, conference]{ieeeconf}  	
\usepackage{etex}
\usepackage{refcount}
\IEEEoverridecommandlockouts                   				
\usepackage{etoolbox}					
\usepackage{tabularx}					
\usepackage{booktabs}					
\usepackage{pbox}						
\usepackage{comment}					
\usepackage[dvipsnames]{xcolor}			
\usepackage{datetime}					
\usepackage{multicol}					
\usepackage{pgffor}						
\usepackage{url}						
\usepackage{setspace}
\usepackage{varwidth}

\usepackage{graphicx}					
\usepackage{cite}						
\usepackage[space]{grffile}				
\usepackage{afterpage}					
\usepackage[noend]{algpseudocode}		
\usepackage{algorithm}					
\usepackage{hyperref}					
\usepackage{fixltx2e}					
\usepackage{microtype}					

\usepackage{mymath}						
\usepackage{mytext}						

\makeatletter
\@ifclassloaded{article}{
	\usepackage{parskip}					
	\usepackage[margin=1in]{geometry}		
	\usepackage[labelfont=bf,font={small}]{caption}	
	\usepackage{subcaption}					

	\bibliographystyle{unsrt}				
	\setcounter{secnumdepth}{4}				
	\setcounter{tocdepth}{4}				
}{
}
\@ifclassloaded{AAS}{
	\usepackage{footnpag}			      	

	\providecommand{\subcaption}[1]{\begin{minipage}{\columnwidth}#1\end{minipage}}
	\bibliographystyle{AAS_publication}		
}{
}
\@ifclassloaded{aiaa-tc}{
}{
}
\@ifclassloaded{ieeeconf}{
	\usepackage{subcaption}					
}{
}
\makeatother

\algrenewcommand\algorithmicdo{}
\algrenewcommand\algorithmicthen{}
\algrenewcommand\algorithmicprocedure{}
\MakeRobust{\Call}

\makeatletter
\newcommand{\algindent}[1][1]{%
	\setlength\@tempdima{\algorithmicindent}%
	\hskip\dimexpr#1\@tempdima\relax%
}
\makeatother

\title{An Asymptotically-Optimal Sampling-Based Algorithm\\ for \Bidirectional Motion Planning}
\author{Joseph A.\ Starek$^\ast$, Javier V.\ Gomez$^\dagger$, Edward Schmerling$^\ddagger$, Lucas Janson$^\mathsection$, Luis Moreno$^\dagger$, Marco Pavone$^\ast$%
	\thanks{$^\ast$Dept.\ of Aeronautics and Astronautics, Stanford University, Stanford, CA 94305, \texttt{\{jstarek, pavone\}@stanford.edu}.}
	\thanks{$^\dagger$Dept.\ of Systems Engineering \& Automation, Carlos III University of Madrid, Madrid, Spain, 28911, \texttt{\{jvgomez, moreno\}@ing.uc3m.es}.}
	\thanks{$^\ddagger$Inst.\ for Computational \& Mathematical Engineering, Stanford University, Stanford, CA 94305, \texttt{schmrlng@stanford.edu}.}
	\thanks{$^\mathsection$Dept.\ of Statistics, Stanford University, Stanford, CA 94305, \texttt{ljanson@stanford.edu}.}
	\thanks{This work was supported by NASA under the Space Technology Research Grants Program, Grant NNX12AQ43G.}
}
\graphicspath{%
	{./figures/}%
	{./figures/Variant Comparison/}%
	{./figures/Results/Plots/}%
}
\renewcommand{\bibfiles}{./bib/alias,./bib/main}
\newcommand{\figwidth}{0.48\linewidth}

\newcounter{terminationline}
\newcounter{swaptreesline}

\newtoggle{arXivVersion}
\settoggle{arXivVersion}{true}

\newcommand{\OptimalityCriterion}{Best Path\xspace}		
\newcommand{\ZminCriterion}{Balanced Trees\xspace}		
\newcommand{\ExpandTreeFromNode}{Expand\xspace}			
\newcommand{\InitializeTree}{Initialize\xspace}			
\newcommand{\InsertNewSample}{Insert\xspace}			

\newcommand{\CostVsTime}{CvT}							
\newcommand{\SuccessVsTime}{SRvT}						

\newcommand{\AlgorithmVariants}{SwapFeas, SwapOpt, ZminFeas, ZminOpt}	
\newcommand{\Dimensions}{5, 10}											
\newcommand{\ObstacleCoverages}{0, 25, 50}								

\newcommand{\submeet}{_{\mathrm{meet}}}					
\newcommand{\bidirectional}{bi-directional\xspace}
\newcommand{\Bidirectional}{Bi-directional\xspace}
\newcommand{\bidirectionality}{bi-directionality\xspace}
\newcommand{\state}{configuration\xspace}
\newcommand{\states}{configurations\xspace}

\renewcommand{\baselinestretch}{0.93}

\renewcommand{\vecformat}{\mathb}
\renewcommand{\MPspaceformat}{\mathcal}
\renewcommand{\nonnegativereals}{\reals_{\geq 0}}

\definevector{Vecx}{x}
\definevector{Vecs}{s}
\definevector{Vecy}{y}
\definevector{Vecz}{z}
\definevector{Vecsigma}{\sigma}
\defineset{SetA}{A}
\defineset{SetB}{B}
\defineset{SetE}{E}
\defineset{SampleSet}{S}
\defineset{SetT}{T}
\defineset{SetV}{V}
\defineMPspace{X}{X}
\newcommand{\FrontierSet}{\SetV_{\mathrm{open}}}		
\newcommand{\UnexploredSet}{\SetV_{\mathrm{unvisited}}}	

\begin{document}

\maketitle

\begin{abstract}
	\Bidirectional search is a widely used strategy to increase the success and convergence rates of sampling-based motion planning algorithms. Yet, few results are available that merge both \bidirectional search and asymptotic optimality into existing optimal planners, such as \PRMstar, \RRTstar, and \FMT. The objective of this paper is to fill this gap. Specifically, this paper presents a \bidirectional, sampling-based, asymptotically-optimal algorithm named \Bidirectional \FMT (\BFMT) that extends the Fast Marching Tree (\FMT) algorithm to \bidirectional search while preserving its key properties, chiefly lazy search and asymptotic optimality through convergence in probability. \BFMT performs a two-source, \emph{lazy} dynamic programming recursion over a set of randomly-drawn samples, correspondingly generating two search trees: one in cost-to-come space from the initial \state and another in cost-to-go space from the goal \state. Numerical experiments illustrate the advantages of \BFMT over its unidirectional counterpart, as well as a number of other state-of-the-art planners.
\end{abstract}

\section{Introduction}

Motion planning is the computation of paths that guide systems from an initial \state to a set of goal \state(s) around nearby obstacles, while possibly optimizing an objective function.  The problem has a long and rich history in the field of robotics, and many algorithmic tools have been developed; we refer the interested reader to \cite{SML:06} and references therein.
Arguably, \emph{sampling-based algorithms} are among the most pervasive, widespread planners available in robotics, including the Probabilistic Roadmap algorithm (\PRM)\cite{LEK-PS-JCL-MHO:96}, the Expansive Space Trees algorithm (EST) \cite{DH-JCL-RM:99a, JMP-NB-LEK:04}, and the Rapidly-Exploring Random Tree algorithm (\RRT) \cite{SML-JJK:01}.
Since their development, efforts to improve the ``quality'' of paths led to asymptotically-optimal (AO) variants of \RRT and \PRM, named \RRTstar and \PRMstar, respectively, whereby the cost of the returned solution converges almost surely to the optimum as the number of samples approaches infinity \cite{SK-EF:11, JL-KH:14}.  Many other planners followed, including BIT$^\ast$ \cite{JDG-SSS-TDB:14} and RRT$^{\#}$ \cite{OA-PT:13} to name a few.  Recently, a conceptually different asymptotically-optimal, sampling-based motion planning algorithm, called the Fast Marching Tree (\FMT) algorithm, has been presented in \cite{LJ-MP:13,LJ-ES-AC-ea:15}.  Numerical experiments suggested that \FMT converges to an optimal solution faster than \PRMstar or \RRTstar, especially in high-dimensional \state spaces and in scenarios where collision-checking is expensive.

It is a well-known fact that \emph{\bidirectional} search can dramatically increase the convergence rate of planning algorithms, prompting some authors \cite{YKH-NA:92} to advocate its use for accelerating essentially any motion planning query.  This was first rigorously studied in \cite{IP:69} and later investigated, for example, in \cite{ML-PR:89, AG-HK-RW:06}.  Collectively, the algorithms presented in \cite{IP:69, ML-PR:89, AG-HK-RW:06, YKH-NA:92} belong to the family of \emph{non}-sampling-based approaches and are more or less closely related to a \bidirectional implementation of the Dijkstra Method.  More recently, and not surprisingly in light of these performance gains, \bidirectional search has been merged with the sampling-based approach, with \RRTConnect and SBL representing the most notable examples \cite{JJK-SML:00, GS-JCL:03}.

Though such \bidirectional versions of \RRT and \PRM are probabilistically complete, they do not enjoy optimality guarantees.  The next logical step in the quest for fast planning algorithms is the design of \emph{\bidirectional}, sampling-based, asymptotically-optimal algorithms. To the best of our knowledge, the only available results in this context are \cite{BA-MS:11} and the unpublished work \cite{MJ-AP:13}, both of which discuss \bidirectional implementations of \RRTstar.  Neither work, however, provides a mathematically-rigorous proof of asymptotic optimality starting from first principles.  Accordingly, the objective of this paper is to propose and \emph{rigorously} analyze such an algorithm.

{\em Statement of Contributions}: This paper introduces the \Bidirectional Fast Marching Tree (\BFMT) algorithm.\footnote{The asterisk $^\ast$, pronounced ``star'', is intended to represent asymptotic optimality much like for the \RRTstar and \PRMstar algorithms.}  To the best of the authors' knowledge, this is the first tree-based, asymptotically-optimal \bidirectional sampling-based planner.  \BFMT extends \FMT to \bidirectional search and essentially performs a ``lazy,'' \bidirectional dynamic programming recursion over a set of probabilistically-drawn samples in the free \state space.  The contribution of this paper is threefold.  First, we present the \BFMT algorithm in \cref{sec: Algorithm Description}.  Second, we rigorously prove the asymptotic optimality of \BFMT (under the notion of convergence in probability) and characterize its convergence rate in \cref{sec: BFMT* Proofs}. We note that the convergence rate of \FMT in \cite{LJ-ES-AC-ea:15} is proved only for obstacle-free \state spaces, while we generalize that result to allow for the presence of obstacles. Finally, we perform numerical experiments in \cref{sec:sim} across a number of planning spaces that suggest \BFMT converges to an \emph{optimal} solution at least as fast as \FMT, \PRMstar, and \RRTstar, and sometimes significantly faster.

\section{Problem Definition}
\label{sec: Problem}

Let $\Xspace$ be a $d$-dimensional \state space, and let $\Xobs$ be the obstacle region, such that $\Xspace \setminus \Xobs$ is an open set (we consider $\boundary\Xspace \subset \Xobs$).  Denote the obstacle-free space as $\Xfree = \closure{\Xspace \setminus \Xobs}$, where $\closure{\cdot}$ denotes the closure of a set.  A path planning problem, denoted by a triplet $(\Xfree, \Vecx\subinit, \Vecx\subgoal)$, seeks to maneuver from an initial \state $\Vecx\subinit$ to a goal \state $\Vecx\subgoal$ through $\Xfree$.  Let a continuous function of \emph{bounded variation} $\map{\sigma}{\unitinterval}{\Xspace}$, called a \emph{path}, be \emph{collision-free} if $\sigma(\tau) \in \Xfree$ for all $\tau \in \unitinterval$.  A path is called a \emph{feasible} solution to the planning problem $(\Xfree, \Vecx\subinit, \Vecx\subgoal)$ if it is collision-free, $\sigma(0) = \Vecx\subinit$, and $\sigma(1) = \Vecx\subgoal$.

Let $\Sigma$ be the set of all paths. A cost function for the planning problem $(\Xfree, \Vecx\subinit, \Vecx\subgoal)$ is a function $\map{\CostFcn}{\Sigma}{\nonnegativereals}$ from $\Sigma$ to the nonnegative real numbers; in this paper, we consider as 
$\CostFcn(\sigma)$ the \emph{arc length} of $\sigma$ with respect to the Euclidean metric in $\Xspace$ (the extension to general cost functions will be briefly discussed in \cref{subset:bfmt_disc}).
\begin{quote}{\bf Optimal path planning problem}: 
	Given a path planning problem $(\Xfree, \Vecx\subinit, \Vecx\subgoal)$ and an arc length function $\map{\CostFcn}{\Sigma}{\nonnegativereals}$, find a feasible path $\sigma\supopt$ such that $\CostFcn(\sigma\supopt) = \min\left\{ \CostFcn(\sigma) \suchthat \sigma \text{ is feasible} \right\}$.  If no such path exists, report failure.
\end{quote}

Finally, we introduce some definitions concerning the \emph{clearance} of a path, \textit{i.e.}, its ``distance" from $\Xobs$ \cite{LJ-ES-AC-ea:15}.  
For a given $\delta > 0$, the $\delta$-interior of $\Xfree$ is defined as the set of all points that are at least a distance $\delta$ away from
any point in $\Xobs$. A collision-free path $\sigma$  is said to have strong $\delta$-clearance if it lies entirely inside the $\delta$-interior of $\Xfree$.
A path planning problem with optimal path cost $\CostFcn\supopt$ is called \emph{$\delta$-robustly feasible} if there exists a strictly positive sequence $\delta_n \rightarrow 0$, with $\delta_n \le \delta \;\, \forall n \in \naturals$, and a sequence $\{\sigma_n\}_{n=1}^{\infty}$ of feasible paths such that $\liminfty[n] \CostFcn(\sigma_n) = \CostFcn\supopt$ and for all $n \in \naturals$, $\sigma_n$ has strong $\delta_n$-clearance, $\sigma_n(1) = \Vecx\subgoal$, $\sigma_n(\tau) \neq \Vecx\subgoal$ for all $\tau \in (0,1)$, and $\sigma_n(0) = \Vecx\subinit$.

\section{The \texorpdfstring{\BFMT}{BFMT*} Algorithm}
\label{sec: Algorithm Description}
In this section, we present the Bi-Directional Fast Marching Tree algorithm, \BFMT, represented in pseudocode as \cref{alg: Static BiDirFMT}.  To begin, we provide a high-level description of \FMT in \cref{subsec:fmt}, on which \BFMT is based.  We follow in \cref{subsec:bfmt} with \BFMT's own high-level description, and then provide additional details in \cref{subsec:bfmt_det}. 

\subsection{\texorpdfstring{\FMT}{FMT*} -- High-level description}
\label{subsec:fmt}
The \FMT algorithm, introduced in \cite{LJ-MP:13,LJ-ES-AC-ea:15}, is a unidirectional algorithm that essentially performs a forward dynamic programming recursion over a set of sampled points and correspondingly generates a \emph{tree of paths} that grow steadily outward in cost-to-come space. The recursion performed by \FMT is characterized by three key features: (1) It is tailored to disk-connected graphs, where two samples are considered \emph{neighbors} (hence connectable) if their distance is below a given bound, referred to as the \emph{connection radius}; (2) It performs graph construction and graph search \emph{concurrently}; and (3) For the evaluation of the immediate cost in the dynamic programming recursion, one ``lazily" ignores the presence of obstacles, and whenever a locally-optimal (assuming no obstacles) connection to a new sample intersects an obstacle, that sample is simply skipped and left for later (as opposed to looking for other locally-optimal connections in the neighborhood).

The last feature, which makes the algorithm ``lazy," may cause \emph{suboptimal} connections. A central property of \FMT is that the
cases where a suboptimal connection is made become vanishingly rare as the number of samples goes to infinity, which helps maintain the algorithm's asymptotically optimality.  This manifests itself into a key computational advantage---by restricting collision detection to only locally-optimal connections, \FMT (as opposed to, \textit{e.g.}, \PRMstar \cite{SK-EF:11}) avoids a large number of costly collision-check computations, at the price of a vanishingly small ``degree" of suboptimality.  We refer the reader to \cite{LJ-MP:13,LJ-ES-AC-ea:15} for a detailed description of the algorithm and its advantages.

\subsection{\texorpdfstring{\BFMT}{BFMT*} -- High-level description}
\label{subsec:bfmt}
At its core, \BFMT implements a \emph{\bidirectional} version of the \FMT algorithm by simultaneously propagating two wavefronts (henceforth, the leaves of an expanding tree will be referred to as the wavefront of the tree) through the free \state space.  \BFMT, therefore, performs a \emph{two-source} dynamic programming recursion over a set of sampled points, and correspondingly generates a \emph{pair} of search trees: one in cost-to-come space from the initial \state and another in cost-to-go space from the goal \state (see \cref{fig: BiDirFMT_Paths}).  Throughout the remainder of the paper, we refer to the former as the \emph{forward tree}, and to the latter as the \emph{backward tree}.


\begingroup
\renewcommand{\AlgorithmVariants}{SwapFeas}
\foreach \variant in \AlgorithmVariants {
\begin{figure}[ht!]
	\centering
	\foreach \coverage in \ObstacleCoverages {
		\begin{subfigure}[c]{0.15\textwidth}
			\centering\includegraphics[width=1.0\textwidth]{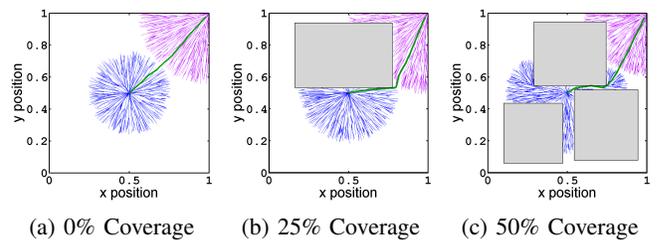}
			\caption{\coverage\% Coverage}
		\end{subfigure}
	}
	\caption{The \BFMT algorithm generates a \emph{pair} of search trees: one in cost-to-come space from the initial \state (blue) and another in cost-to-go space from the goal \state (purple). The path found by the algorithm is in green color.}
	\label{fig: BiDirFMT_Paths}
\end{figure}
}
\endgroup


The dynamic programming recursion performed by \BFMT is characterized by the same lazy feature of \FMT (see \cref{subsec:fmt}).  However, the time it takes to run \BFMT on a given number of samples can be substantially smaller than for \FMT. Indeed, for uncluttered \state spaces, the search trees grow hyperspherically, and hence \BFMT only has to expand about half as far (in both trees) as \FMT in order to return a solution. This is made clear in \cref{fig: BiDirFMT_Paths}(a), in which \FMT would have to expand the forward tree twice as far to find a solution. Since runtime scales approximately with edge number, which scales as the linear distance covered by the tree raised to the dimension of the state space, we may expect in loosely cluttered \state spaces an approximate speed-up of a factor $2^{d-1}$ over \FMT in $d$-dimensional space (the $-1$ in the exponent is because \BFMT has to expand 2 trees, so it loses one factor of 2 advantage).

\subsection{\texorpdfstring{\BFMT}{BFMT*} -- Detailed description}
\label{subsec:bfmt_det}

To understand the \BFMT algorithm, some background notation must first be introduced.  Let $\SampleSet$ be a set of points sampled independently and identically from the uniform distribution on $\Xfree$, to which $\Vecx\subinit$ and $\Vecx\subgoal$ are added. (The extension to non-uniform sampling distributions is addressed in \cref{subset:bfmt_disc}.)  Let tree $\SetT$ be the quadruple $(\SetV, \SetE, \UnexploredSet, \FrontierSet)$, where $\SetV$ is the set of tree nodes, $\SetE$ is the set of tree edges, and $\UnexploredSet$ and $\FrontierSet$ are mutually exclusive sets containing the \emph{unvisited} samples in $\SampleSet$ and the \emph{wavefront} nodes in $\SetV$, correspondingly.  To be precise, the unvisited set $\UnexploredSet$ stores all samples in the sample set $\SampleSet$ that have not yet been considered for addition to the tree of paths.  The wavefront set $\FrontierSet$, on the other hand, tracks in sorted order (by cost from the root) only those nodes which have already been added to the tree that are near enough to tree leaves to actually form better connections.  These sets play the same role as their counterparts in \FMT, see \cite{LJ-MP:13,LJ-ES-AC-ea:15}.  However, in this case \BFMT ``grows'' two such trees, referred to as $\SetT = (\SetV, \SetE,  \UnexploredSet, \FrontierSet)$ and $\SetT^\prime = (\SetV^\prime, \SetE^\prime, \UnexploredSet^\prime, \FrontierSet^\prime)$.  Initially, $\SetT$ is the tree rooted at $\Vecx\subinit$, while $\SetT'$ is the tree rooted at $\Vecx\subgoal$.  Note, however, that the trees are exchanged during the execution of \BFMT, so $\SetT$ in \cref{alg: Static BiDirFMT} is not always the tree that contains $\Vecx\subinit$.

The \BFMT algorithm is represented in \cref{alg: Static BiDirFMT}.  Before describing \BFMT in detail, we list briefly the basic planning functions employed by the algorithm.  Let \Call{SampleFree}{$n$} be a function that returns a set of $n \in \naturals$ points sampled independently and identically from the uniform distribution on $\Xfree$.  Let $\Call{Cost}{\overline{\tilde{\Vecx} \Vecx}}$ be the cost of the straight-line path between \states $\tilde{\Vecx}$ and $\Vecx$. Let \Call{Path}{$\Vecz, \SetT$} return the unique path in tree $\SetT$ from its root to node $\Vecz$.  Also, with a slight abuse of notation, let \Call{Cost}{$\Vecx$, $\SetT$} return the cost of the unique path in tree $\SetT$ from its root to node $\Vecx$, and let \Call{CollisionFree}{$\Vecx$, $\Vecy$} be a boolean function returning true if the straight-line path between \states $\Vecx$ and $\Vecy$ is collision free.  Given a set of samples $\SetA$, let \Call{Near}{$\SetA$, $\Vecz$, $r$} return the subset of $\SetA$ within a ball of radius $r$ centered at sample $\Vecz$ (\textit{i.e.}, the set $\left\{\Vecx \in \SetA \suchthat \norm{\Vecx - \Vecz} < r \right\}$).  Let the \Call{Terminate}{} function represent an external termination criterion (\textit{i.e.}, timeout, maximum number of samples, etc.) which can be used to force early termination (or prevent infinite runtime for infeasible problems).  Finally, regarding tree expansion, let \Call{Swap}{$\SetT$, $\SetT'$} be a function that swaps the two trees $\SetT$ and $\SetT^\prime$. 
and let \Call{Companion}{$\SetT$} return the companion tree $\SetT'$ to $\SetT$ (or vice versa).

We are now in position to describe the \BFMT algorithm.  First, a set of $n$ \states in $\Xfree$ is determined by drawing samples uniformly.  Two trees are then initialized using the \Call{\InitializeTree}{} subfunction at the bottom of \cref{alg: Static BiDirFMT}, with a forward tree rooted at $\Vecx\subinit$ and a reverse tree rooted at $\Vecx\subgoal$.  Once complete, tree expansion begins starting with tree $\SetT$ rooted at $\Vecx\subinit$ using the \Call{\ExpandTreeFromNode}{} procedure in \cref{alg: ExpandTreeFromNode}.  In the following, the node selected for expansion will be consistently denoted by $\Vecz$, while $\Vecx\submeet$ will denote the lowest-cost candidate node for tree connection (\textit{i.e.}, for joining the two trees). The \Call{\ExpandTreeFromNode}{} procedure requires the specification of a  \emph{connection radius} parameter, $r_n$, whose selection will be discussed in \cref{sec: BFMT* Proofs}. \Call{\ExpandTreeFromNode}{} implements the ``lazy'' dynamic programming recursion described (at a high level) in \cref{subsec:bfmt}, making locally-optimal collision-free connections \emph{from} nodes $\Vecx$ near $\Vecz$ unvisited by tree $\SetT$ (those in set $\UnexploredSet$ within search radius $r_n$ of $\Vecz$) \emph{to} wavefront nodes $\Vecx^\prime$ near each $\Vecx$ (those in set $\FrontierSet$ within search radius $r_n$ of $\Vecx$).  Any collision-free edges and newly-connected nodes found are then added to $\SetT$, the connection candidate node  $\Vecx\submeet$ is updated, and $\Vecz$ is dropped from the list  of wavefront nodes.  The key feature of the \Call{\ExpandTreeFromNode}{} function is that in the execution of the dynamic programming recursion it ``lazily'' ignores the presence of obstacles (see \cref{line:lazy}) -- as discussed in \cref{sec: BFMT* Proofs} this comes at no loss of (asymptotic) optimality (see also \cite{LJ-MP:13,LJ-ES-AC-ea:15}). Note the \Call{\ExpandTreeFromNode}{} function is identical to that of unidirectional \FMT, with the exception here of additional lines for tracking connection candidate $\Vecx\submeet$.

After expansion, the algorithm checks whether a feasible path is found on \cref{line: feasibility}.  If unsuccessful so far, \Call{Terminate}{} (which reports failure upon early termination) is checked before proceeding.
If the algorithm has not terminated, it checks whether the wavefront of the companion tree is empty (\cref{line: BiDirFMT Frontier Check}). If this is the case, the  \Call{\InsertNewSample}{} function shown in \cref{alg: InsertNewSample} samples a new \state $\Vecs$ uniformly from $\Xfree$ and tries to connect it to a nearest neighbor in the companion tree within radius $r_n$.  This way, the expanding tree is ensured to have at least one \state in its wavefront available for expansion on subsequent iterations (the alternative would be to report failure).  This mimics anytime behavior, and by forcing samples to lie close to tree nodes we effectively ``reopen'' closed nodes for expansion again.  Uniform resampling may require many attempts before finding a \state $\Vecs$ which can be successfully connected to $\FrontierSet^\prime$, though this appeared to have a negligible impact on running time for our path planning studies.  On the other hand, a more effective strategy might bias resampling towards areas requiring expansion (\textit{e.g.}, bottlenecks, traps) rather than uniformly within tree coverage.

The algorithm then proceeds on \crefrange{line: swap1}{line: swap2} with the selection of the next node (and corresponding tree) for expansion.  As shown, \BFMTstar ``swaps'' the forward and backward trees on each iteration, each being expanded in turns.  As \Call{\InsertNewSample}{} ensures the companion tree $\SetT^\prime$ always has at least one node in its frontier $\FrontierSet^\prime$, a node is always available for subsequent expansion as the next $\Vecz$.
After selection, the entire process is iterated.

\begingroup
\setlength{\columnsep}{0.5cm}
\setlength{\columnseprule}{0.2pt}
\setlength{\multicolsep}{2.0pt plus 2.0pt minus 1.5pt}
\newcommand{\BalancedTreesLineOne}{\mbox{\small$\displaystyle\Vecz_1 \leftarrow \argmin_{\Vecx \in \FrontierSet}\{\Call{Cost}{\Vecx, \SetT}\}$}}
\newcommand{\BalancedTreesLineTwo}{\mbox{\small$\displaystyle\Vecz_2 \leftarrow \argmin_{\Vecx^\prime \in \FrontierSet^\prime}\{\Call{Cost}{\Vecx^\prime, \SetT^\prime}\}$}}
\newcommand{\BalancedTreesLineThree}{\mbox{\small$\displaystyle (\Vecz,\!\SetT)\!\! \leftarrow\!\! \!\!\!\!\!\!\argmin_{(\Vecz_1, \SetT),(\Vecz_2, \SetT')} \!\!\!\!\!\left\{\Call{Cost}{\Vecz_i, \SetT_i}\right\}$}}
\newcommand{\BalancedTreesLineFour}{$\SetT^\prime = $ \Call{Companion}{$\SetT$}}
\newcommand{\AlternatingTreesLineOne}{\mbox{\small$\displaystyle\!\!\!\!\!\!\!\!\!\!\Vecz\!\leftarrow\!\argmin_{\Vecx^\prime \in \FrontierSet^\prime}\{\Call{Cost}{\Vecx^\prime, \SetT^\prime}\}\!\!\!\!\!\!\!\!\!\!$}}
\newcommand{\AlternatingTreesLineTwo}{\mbox{$\displaystyle\!\!\!\!\!\!\!\!\!\!\Call{Swap}{\SetT, \SetT^\prime}$}}
\newcommand{\subtmp}{_{\textrm{tmp}}}

\newcommand\algsubroutine[1]{%
	\makeatletter\setcounter{ALG@line}{0}\makeatother
	\ifNoText{#1}{}{
		\vspace*{-.7\baselineskip}\Statex\hspace*{\dimexpr-\algorithmicindent-2pt\relax}\rule{\columnwidth}{0.4pt}%
		\Statex\hspace*{-\algorithmicindent}\textbf{#1}%
	}
	\vspace*{-.7\baselineskip}\Statex\hspace*{\dimexpr-\algorithmicindent-2pt\relax}\rule{\columnwidth}{0.4pt}%
}

\begin{algorithm}[th!]
	\caption{The \Bidirectional Fast Marching Tree \\ Algorithm (\BFMT)}
	\label{alg: Static BiDirFMT}
	\begin{algorithmic}[1]
		\small
		\Require Query $\left(\Vecx\subinit, \Vecx\subgoal\right)$, Search radius $r_n$, Sample count $n$
		\State $\SampleSet \leftarrow \{\Vecx\subinit, \Vecx\subgoal\} \union \Call{SampleFree}{n}$
		\State $\SetT \leftarrow$ \Call{\InitializeTree}{$\SampleSet$, $\Vecx\subinit$}
		\State $\SetT^\prime \leftarrow$ \Call{\InitializeTree}{$\SampleSet$, $\Vecx\subgoal$}
		\State $\Vecz \leftarrow \Vecx\subinit$, $\Vecx\submeet \leftarrow \nullset$, $\sigma\supopt \leftarrow \nullset$
		\While{$\sigma\supopt = \nullset$} 
			\State $\left\{\Vecx\submeet, \SetT\right\} \leftarrow$ \Call{\ExpandTreeFromNode}{$\SetT$, $\Vecz$, $r_n$, $\Vecx\submeet$}
			\State \textbf{if} $\Vecx\submeet \not= \nullset$ \label{line: feasibility}
				\State \algindent[1] $\sigma\supopt \leftarrow$ \Call{Path}{$\Vecx\submeet, \SetT$} $\union$ \Call{Path}{$\Vecx\submeet, \SetT^\prime$}
				\State \algindent[1] \textbf{break} 
			\State \textbf{else if} \Call{Terminate}{{}} \textbf{then} \textbf{return} Failure
			\State \textbf{else if} $\FrontierSet^\prime = \nullset$ \textbf{then} $\SetT^\prime \leftarrow \Call{\InsertNewSample}{\SetT^\prime, r_n}$ \label{line: BiDirFMT Frontier Check} 				
			\State \algindent[1] \AlternatingTreesLineOne		\label{line: swap1}
			\State \algindent[1] \AlternatingTreesLineTwo		\label{line: swap2}
		\EndWhile
		\State \Return $\sigma\supopt$
		\algsubroutine{}
		\Function{\InitializeTree}{$\SampleSet$, $\Vecx_0$}
			\State $\SetV \leftarrow \nullset$, $\SetE \leftarrow \nullset$, $\UnexploredSet \leftarrow \SampleSet$, $\FrontierSet \leftarrow \nullset$
			\State \Return $\SetT \leftarrow \Call{AddNode}{(\SetV, \SetE, \UnexploredSet, \FrontierSet), \Vecx_0}$
		\EndFunction
		\algsubroutine{}
		\Function{AddNode}{$\SetT, \Vecx$}
			\State $\SetV \leftarrow \SetV \,\union\, \{\Vecx\}$ \Comment{Add $\Vecx$ to tree}
			\State $\SetE \leftarrow \SetE \,\union\, \{(\Vecx\submin, \Vecx)\}$ \Comment{Add edge to tree}
			\State $\UnexploredSet \leftarrow \UnexploredSet \!\setminus\! \{\Vecx\}$ \Comment{Mark $\Vecx$ as visited}
			\State $\FrontierSet \leftarrow \FrontierSet \union \{\Vecx\}$ \Comment{Add $\Vecx$ to wavefront}
			\State \Return $\SetT \leftarrow (\SetV, \SetE, \UnexploredSet, \FrontierSet)$
		\EndFunction
	\end{algorithmic}
\end{algorithm}
\begin{figure}
	\begin{minipage}[t]{0.98\linewidth}
		\vspace*{-0.42cm}
		\begin{algorithm}[H]
			\caption{Fast Marching Tree Expansion Step}
			\label{alg: ExpandTreeFromNode}
			\begin{algorithmic}[1]
				\small
				\Function{\ExpandTreeFromNode}{$\SetT$,	$\Vecz$, $r_n$, $\Vecx\submeet$}
					\State $\FrontierSet{}_{,}{}\subnew \leftarrow \nullset$
					\State $Z\subnear \leftarrow \Call{Near}{\UnexploredSet, \Vecz, r_n}$
					\For{$\Vecx \in Z\subnear$}
						\State $X\subnear \leftarrow \Call{Near}{\FrontierSet, \Vecx, r_n}$
						\State $\displaystyle \Vecx\submin \leftarrow \argmin_{\tilde{\Vecx} \in X\subnear}\{\Call{Cost}{\tilde{\Vecx}, \SetT} \!+\! \Call{Cost}{\overline{\tilde{\Vecx} \Vecx}}\}$  				\label{line:lazy}
						\If{$\Call{CollisionFree}{\Vecx\submin, \Vecx}$}  	\label{line:collisionfreecheck}
							\State \mbox{$(\SetV, \SetE, \UnexploredSet, \FrontierSet{}_{,}{}\subnew) \leftarrow$}
							\Statex \algindent[4] $\Call{AddNode}{(\SetV, \SetE, \UnexploredSet, \FrontierSet{}_{,}{}\subnew), \Vecx}$
							\If{\begin{varwidth}[t]{0.77\linewidth} $\left\{\right.\Vecx \in \SetV^\prime$ and $\Call{Cost}{\Vecx, \SetT} + \Call{Cost}{\Vecx, \SetT^\prime} < \Call{Cost}{\Vecx\submeet, \SetT} + \Call{Cost}{\Vecx\submeet, \SetT^\prime}\left.\right\}$ \end{varwidth}}
								\State $\vphantom{\Bigl(\Bigr)}$ $\Vecx\submeet \!\leftarrow\! \Vecx$ \Comment{Save $\Vecx$ as best connection}
							\EndIf
						\EndIf
					\EndFor
					\State $\FrontierSet \leftarrow \left(\FrontierSet \,\union\, \FrontierSet{}_{,}{}\subnew\right) \!\setminus\! \{\Vecz\}$ \Comment{Add new nodes\\\hfill to the wavefront; drop $\Vecz$ from the wavefront}
					\State \Return $\left\{ \Vecx\submeet, \SetT \leftarrow (\SetV, \SetE, \UnexploredSet, \FrontierSet) \right\}$
				\EndFunction
			\end{algorithmic}
		\end{algorithm}%
		\vspace*{-0.50cm}
		\begin{algorithm}[H]
			\caption{Insertion of New Samples}
			\label{alg: InsertNewSample}
			\begin{algorithmic}[1]
				\small
				\Function{\InsertNewSample}{$\SetT$, $r_n$}
					\While{$\FrontierSet = \nullset$ \textbf{and not} \Call{Terminate}{{}}}
					  \State $\Vecs \leftarrow \Call{SampleFree}{1}$ \label{line:sample one}
						\State $V\subnear \leftarrow \Call{Near}{\SetV, \Vecs, r_n}$
						\While{$V\subnear \neq \nullset$}
							\State $\displaystyle \Vecx\submin \leftarrow \argmin_{\Vecx \in V\subnear}\{\Call{Cost}{\Vecx, \SetT} \!+\! \Call{Cost}{\overline{\Vecx \Vecs}}\}$
							\If{\Call{CollisionFree}{$\Vecx\submin$, $\Vecs$}}
								\State $\SetT \leftarrow \Call{AddNode}{\SetT, \Vecs}$
								\State \textbf{break}
							\Else \textbf{ then} $V\subnear \leftarrow V\subnear \setminus \{\Vecx\submin\}$
							\EndIf
						\EndWhile
					\EndWhile
					\State \Return $\SetT \leftarrow (\SetV, \SetE, \UnexploredSet, \FrontierSet)$
				\EndFunction
			\end{algorithmic}
		\end{algorithm}
	\end{minipage}
\end{figure}

\endgroup

\subsubsection{\texorpdfstring{\BFMT}{BFMT*} -- Variations}
\label{subsubsec: BFMT-Variations}
As for any \bidirectional planner, the correctness and computational efficiency of \BFMT hinge upon two key aspects: (i) how computation is interleaved among the two trees (in other words, which wavefront at each step should be chosen for expansion), and (ii) when the algorithm should terminate.  For instance, as an alternative tree expansion strategy (\textit{i.e.}, item (i)), one could replace \crefrange{line: swap1}{line: swap2} with the ``\MakeLowercase{\ZminCriterion}'' condition which enforces more of a balanced search, maintaining equal costs from the root within each wavefront such that the two wavefronts propagate and meet roughly equidistantly in cost-to-go from their roots:
\bgroup
\setcounterref{terminationline}{line: feasibility} \addtocounter{terminationline}{-1}
\setcounterref{swaptreesline}{line: swap1} \addtocounter{swaptreesline}{-1}
\newfloat{myalgsnippet}{tbhp}{myalgsnippet}%
\newenvironment{algorithmSnippet}[2][H]%
{\begin{myalgsnippet}[#1]\centering\begin{minipage}[t]{#2}\begin{algorithm}[H]}%
{\end{algorithm}\end{minipage}\end{myalgsnippet}}
\floatstyle{plaintop}		
\restylefloat{myalgsnippet}\restylefloat{algorithm}
\addtolength{\intextsep}{-0.75em}%
\setlength{\abovecaptionskip}{0pt}%
\setlength{\belowcaptionskip}{0pt}%
\begin{algorithmSnippet}[H]{0.8\columnwidth}
	\begin{algorithmic}[1]
		\makeatletter\setcounter{ALG@line}{\theswaptreesline}\makeatother
		\State $\displaystyle \Vecz_1 \leftarrow \argmin_{\Vecx \in \FrontierSet}\{\Call{Cost}{\Vecx, \SetT}\}$
		\State $\displaystyle \Vecz_2 \leftarrow \argmin_{\Vecx^\prime \in \FrontierSet^\prime}\{\Call{Cost}{\Vecx^\prime, \SetT^\prime}\}$
		\State $\displaystyle (\Vecz,\SetT) \leftarrow \argmin_{(\Vecz_1, \SetT),(\Vecz_2, \SetT')} \left\{\Call{Cost}{\Vecz_i, \SetT_i}\right\}$
		\State $\SetT^\prime = \Call{Companion}{\SetT}$
	\end{algorithmic}
\end{algorithmSnippet}
\noindent
Similarly, as an alternative termination condition (\textit{i.e.}, item (ii)), one might replace \cref{line: feasibility} with the ``\MakeLowercase{\OptimalityCriterion}'' criterion:
\begin{algorithmSnippet}[H]{0.8\columnwidth}
	\begin{algorithmic}[1]
		\makeatletter\setcounter{ALG@line}{\theterminationline}\makeatother
		\State $\Vecz \in \left(\SetV^\prime \setminus \FrontierSet^\prime\right)$
	\end{algorithmic}
\end{algorithmSnippet}
\egroup
\noindent
Currently \cref{line: feasibility} returns the first available path discovered, at the moment that the two wavefronts touch at $\Vecx\submeet$ (which is not, in general, the lowest cost path).  This alternative condition, on the other hand, returns the \emph{exact optimal path} from $\Vecx\subinit$ to $\Vecx\subgoal$ through the given set $\SampleSet$ of $n$ samples.  This change terminates \BFMT when the two wavefronts have propagated sufficiently far through each other that no better solution can be discovered.  Intuitively-speaking, this occurs at the first moment where the two trees have both selected, at the current iteration or previously, the same node as the minimum cost node $\Vecz$ from their respective roots.

Though seemingly promising ideas, no appreciable differences in performance were found using the above criteria in combination or otherwise; hence we report only the simplest version of our planner as \cref{alg: Static BiDirFMT}.

\section{Asymptotic Optimality of \texorpdfstring{\BFMT}{BFMT*}}
\label{sec: BFMT* Proofs}

In this section, we prove the asymptotic optimality of \BFMT.
We begin with a result called \emph{probabilistic exhaustivity} that essentially states that any path in $\Xfree$ may be ``traced'' arbitrarily well by connecting randomly-distributed points from a sufficiently large sample set covering $\Xfree$.  We then prove the (asymptotic) optimality of \BFMT by showing that it returns solutions with costs no greater than that of any tracing path.  The claim is proven assuming \BFMT acts without the \Call{\InsertNewSample}{} procedure (\cref{alg: InsertNewSample}), in place of which ``Failure'' is reported instead.  The proof for the full algorithm then follows immediately by \emph{a fortiori} argument.

\subsection{Probabilistic exhaustivity}
Let $\map{\sigma}{\unitinterval}{\Xspace}$ be a path.  Given a set of samples (referred to as waypoints) $\left\{\Vecy_m\right\}_{m=1}^M \subset \Xspace$, we associate a path $\map{y}{\unitinterval}{\Xspace}$ that sequentially connects the nodes $\Vecy_1,\dots,\Vecy_M$ with line segments.  We consider the waypoints $\{\Vecy_m\}$ to \emph{$(\eps, r)$-trace} the path $\sigma$ if: (i) $\norm{\Vecy_m - \Vecy_{m+1}} \leq r$ for all $m$, (ii) the cost of $y$ is bounded as $\CostFcn(y) \leq (1+\eps)\CostFcn(\sigma)$, and (iii) the distance from any point of $y$ to $\sigma$ is no more than $r$, \textit{i.e.}, $\min_{t \in \unitinterval} \norm{y(s) - \sigma(t)} \leq r$ for all $s \in \unitinterval$.  In the context of sampling-based motion planning, we may expect to find closely-tracing $\left\{\Vecy_m\right\}$ as a subset of the sampled points, provided the sample size is large.
\iftoggle{arXivVersion}{%
	 This notion is formalized in the following theorem (\cref{thm:pathtracing}), proved as Theorem IV.5 in \cite{ES-LJ-MP:15a} for the general case of driftless control-affine control systems, a special case of which is path planning without differential constraints (as addressed in this paper).
}{%
	This notion is formalized below (\cref{thm:pathtracing}, proven as Theorem IV.5 in \cite{ES-LJ-MP:15a}).
}

\begin{theorem}[Probabilistic exhaustivity]
	\label{thm:pathtracing}
	Define path planning problem $\left(\Xfree, \Vecx\subinit, \Vecx\subgoal\right)$ and let $\map{\sigma}{\unitinterval}{\Xfree}$ be a feasible path.  Denote the volume of the $d$-dimensional Euclidean unit ball by $\zeta_d$.  Finally, let $\SampleSet = \{\Vecx\subinit, \Vecx\subgoal\} \union \Call{SampleFree}{n}$, $\eps > 0$, and for fixed $n$ consider the event $\SetA_n$ that there exist $\left\{\Vecy_m\right\}_{m=1}^M \subset \SampleSet$, $y_1 = \Vecx\subinit$, $y_M = \Vecx\subgoal$ which $(\eps, r_n)$-trace $\sigma$, where
	\begin{align}
		\label{eqn: Connection Radius}
		r_n = 4\, (1 + \eta)^{\frac{1}{d}} \left(\frac{1}{d}\right)^{\frac{1}{d}} \left(\frac{\mu(\Xfree)}{\zeta_d}\right)^{\frac{1}{d}} \left(\frac{\log n}{n}\right)^{\frac{1}{d}}
	\end{align}
	for a parameter $\eta \geq 0$. Then, as $n \to \infty$, the probability that $\SetA_n$ does not occur (denoted by its complement $\SetAComplement_n$) is asymptotically bounded as $\Probability{\SetAComplement_n} = \BigO{n^{-\frac{\eta}{d}} \log^{-\frac{1}{d}} n}$.
\end{theorem}

\subsection{Asymptotic optimality (AO)}
We are now in a position to prove the asymptotic optimality of \BFMT, which represents the main result of this section. We start with an important lemma, which relates the cost of the path returned by \BFMT to that of \emph{any} feasible path. 

\begin{lemma}[\Bidirectional \FMT cost comparison]
	\label{thm: BFMT Cost Comparison}
	Let $\map{\sigma}{\unitinterval}{\Xfree}$ be a feasible path with strong $\delta$-clearance. Consider running \BFMT to completion with $n$ samples and a connection radius $r_n$ given by \cref{eqn: Connection Radius} with $\eta \geq 0$. Let $\CostFcn_n$ denote the cost of the path returned by \BFMT.  Then for fixed $\eps > 0$:
		\[ \Probability{\CostFcn_n > (1 + \eps) \CostFcn(\sigma)} = \BigO{n^{-\frac{\eta}{d}}\log^{-\frac{1}{d}} n}.\]
\end{lemma}

\begin{proof}
	%
	\iftoggle{arXivVersion}{%
	}{%
		For brevity, we have relegated the full details of the proof to an arXiv version of the paper.  A sketch of the proof, however, proceeds as follows.
	}
	Running \BFMT to completion generates one cost-to-come tree $\SetT_i$ and one cost-to-go tree  $\SetT_g$ rooted at $\Vecx\subinit$ and $\Vecx\subgoal$, respectively (subscripts $i$ and $g$ are used to resolve tree root ambiguity).  If $\Vecx\subinit = \Vecx\subgoal$, then \BFMT immediately terminates with $\CostFcn_n = 0$, trivially satisfying the claim. Thus we assume that $\Vecx\subinit \neq \Vecx\subgoal$.  Consider $n$ sufficiently large so that $r_n \leq \min\left\{ \sidefrac{\delta}{2}, \eps \sidefrac{\norm{\Vecx\subinit-\Vecx\subgoal}}{2} \right\}$.  Apply \cref{thm:pathtracing} to produce, with probability at least $1 - \BigO{n^{-\frac{\eta}{d}} \log^{-\frac{1}{d}} n}$, a sequence of waypoints $\{\Vecy_m\}_{m=1}^{M} \subset \SampleSet$, $\Vecy_1 = \Vecx\subinit$, $\Vecy_M = \Vecx\subgoal$ which $(\sidefrac{\eps}{2}, r_n)$-trace $\sigma$.  We claim that in the event that such $\{\Vecy_m\}$ exists, the \BFMT algorithm returns a path with cost upper bounded as $\CostFcn_n \leq \CostFcn(y) + r_n \leq \left(1 + \sidefrac{\eps}{2}\right) \CostFcn(\sigma) + \left(\sidefrac{\eps}{2}\right) \CostFcn(\sigma) = \left(1+\eps\right) \CostFcn(\sigma)$.  The desired result follows directly.
	
	\iftoggle{arXivVersion}{%
		To prove the claim, assume the existence of an $(\sidefrac{\eps}{2}, r_n)$-tracing $\{\Vecy_m\}$ of $\sigma$.  Let $\ball[]{\Vecx}{r}$ represent a ball of radius $r$ centered at a sample $\Vecx$.  Note that our upper bound on $r_n$ implies that $\ball{\Vecy_m}{r_n}$ intersects no obstacles.  This follows from our choice of $r_n$ and the distance bound
		\begin{align*}
			\inf_{\Vecs \in \Xobs} \norm{\Vecy_m - \Vecs} &\geq \inf_{\Vecs \in \Xobs} \norm{\Vecsigma_m - \Vecs} - \norm{\Vecy_m - \Vecsigma_m} \\
				&\geq 2 r_n - r_n \geq r_n.
		\end{align*}
		where $\Vecsigma_m$ is the closest point of $\sigma$ to $\Vecy_m$. This fact, along with $\norm{\Vecy_m - \Vecy_{m+1}} \leq r_n$ for all $m$, implies that when a connection is attempted for $\Vecy_m$, both $\Vecy_{m-1}$ and $\Vecy_{m+1}$ will be in the search radius and no obstacles will lie within that search radius. Running \BFMT to completion generates one cost-to-come tree $\SetT_i\left(\SetV_i, \SetE_i, \FrontierSet{}_{,}{}_i, \UnexploredSet{}_{,}{}_i\right)$ and one cost-to-go tree  $\SetT_g\left(\SetV_g, \SetE_g, \FrontierSet{}_{,}{}_g, \UnexploredSet{}_{,}{}_g\right)$ rooted at $\Vecx\subinit$ and $\Vecx\subgoal$, respectively (the subscripts $i$ and $g$ are used to identify the root of a tree without ambiguity).  The above discussion ensures that the trees will meet and the algorithm will return a feasible path when it terminates -- the path outlined by the waypoints $\{\Vecy_m\}$ disallows the possibility of failure.
		
		For each sample point $\Vecx \in \SampleSet$, let $\CostFcn_i(\Vecx) := \Call{Cost}{\Vecx,\SetT_i}$ denote the cost-to-come of $\Vecx$ from $\Vecx\subinit$ in $\SetT_i$, and let $\CostFcn_g(\Vecx) := \Call{Cost}{\Vecx,\SetT_g}$ denote the cost-to-go from $\Vecx$ to $\Vecx\subgoal$ in $\SetT_g$.  If $\Vecx$ is not contained in a tree $\SetT_k$, $k=\{i,g\}$, we set $\CostFcn_k(\Vecx) = \infty$.  When the algorithm terminates, we know there exists a sample point $\Vecx\submeet \in \SetV_i \intersect \SetV_g$ where the two trees meet; indeed we select the particular meeting point $\Vecx\submeet = \argmin_{\Vecx \in \SetV_i \intersect \SetV_g} \CostFcn_i(\Vecx) + \CostFcn_g(\Vecx)$. Then $\CostFcn_n = \CostFcn_i(\Vecx\submeet) + \CostFcn_g(\Vecx\submeet)$.  We now note a lemma bounding the costs-to-come of the $\{\Vecy_m\}$, the proof of which may be found as an inductive hypothesis (Eq. 5) in Theorem VI.1 of \cite{ES-LJ-MP:15a}.
		\begin{lemma}
			\label{lem:ctocym}
			Let $m \in \{1,\dots,M\}$.  If $\CostFcn_i(\Vecy_m) < \infty$, then $\CostFcn_i(\Vecy_m) \leq \sum_{k=1}^{m-1} \norm{\Vecy_k - \Vecy_{k+1}}$.  Otherwise if $\Vecy_m \notin \SetV_i$, then $\CostFcn_i(\Vecx\submeet) \leq \sum_{k=1}^{m-1} \norm{\Vecy_k - \Vecy_{k+1}}$.  Similarly if $\CostFcn_g(\Vecy_m) < \infty$, then $\CostFcn_g(\Vecy_m) \leq \sum_{k=m}^{M-1} \norm{\Vecy_k - \Vecy_{k+1}}$; otherwise $\CostFcn_g(\Vecx\submeet) \leq \sum_{k=m}^{M-1} \norm{\Vecy_k - \Vecy_{k+1}}$.
		\end{lemma}
		
		To bound the performance $\CostFcn_n$ of \BFMT, there are two cases to consider.  Note in either case we find that $\CostFcn_n \leq \CostFcn(\Vecy) + r_n$, thus completing the proof.
		
		\begin{quote}
			\underline{Case 1}: There exists some $\Vecy_m \in \SetV_i \intersect \SetV_g$.\\
			In this case, $\CostFcn_n = \CostFcn_i(\Vecx\submeet) + \CostFcn_g(\Vecx\submeet) \leq \CostFcn_i(\Vecy_m) + \CostFcn_g(\Vecy_m) < \infty$ by our choice of $\Vecx\submeet$. Then applying \cref{lem:ctocym} we see that $\CostFcn_n \leq \CostFcn_i(\Vecy_m) + \CostFcn_g(\Vecy_m) \leq \sum_{k=1}^{M-1} \norm{\Vecy_k - \Vecy_{k+1}} = \CostFcn(\Vecy)$.
		\end{quote}
		\vskip1em
		\begin{quote}
			\underline{Case 2}: There are no $\Vecy_m \in \SetV_i \intersect \SetV_g$. \\
			Consider $\widetilde m = \max\left\{m \suchthat \CostFcn_i(\Vecy_m) < \infty\right\}$. Then $\Vecy_{\widetilde m} \in \SetV_i$ and $\Vecy_{\widetilde m}$ can not have been the minimum cost element of $\FrontierSet{}_{,}{}_i$ at any point during algorithm execution or else we would have connected $\Vecy_{\widetilde m + 1} \in \SetV_i$. Let $\Vecz$ denote the minimum cost element of $\FrontierSet{}_{,}{}_i$ when $\Vecx\submeet$ was added to $\SetV_i$. We have the bound:
			\begin{align}
				\label{eqn:cmeetbound}
				\CostFcn_i(\Vecx\submeet) &\leq \CostFcn_i(\Vecz) + r_n \!\leq\! \CostFcn_i(\Vecy_{\widetilde m}) + r_n \notag\\
					&\leq \sum_{k=1}^{m-1} \norm{\Vecy_k \!-\! \Vecy_{k+1}} + r_n.
			\end{align}
			By our assumption for this case, $\Vecy_{\widetilde m} \notin \SetV_g$. Then by \cref{lem:ctocym} we know that $\CostFcn_g(\Vecx\submeet) \leq \sum_{k=m}^{M-1} \norm{\Vecy_k - \Vecy_{k+1}}$.  Combining with the previous inequality yields $\CostFcn_n = \CostFcn_i(\Vecx\submeet) + \CostFcn_g(\Vecx\submeet) \leq \sum_{k=1}^{M-1} \norm{\Vecy_k - \Vecy_{k+1}} + r_n = \CostFcn(y) + r_n$.
		\end{quote}
	}{%
		To prove the claim, assume the existence of an $(\sidefrac{\eps}{2}, r_n)$-tracing $\{\Vecy_m\}$.  Let $\ball[]{\Vecx}{r}$ represent a ball of radius $r$ centered at a sample $\Vecx$.  From our tracing assumption and the upper bound on $r_n$, one can show that $\ball{\Vecy_m}{r_n}$ (and hence each $\Vecy_{m-1}$--$\Vecy_m$ and $\Vecy_m$--$\Vecy_{m+1}$ connection) lies in $\Xfree$.  The existence of a collision-free path through $\{\Vecy_m\}$ ensures that trees $\SetT_i$ and $\SetT_g$ will meet and hence that \BFMT will return a feasible path at termination.

		For each sample point $\Vecx \in \SampleSet$, let $\CostFcn_k(\Vecx) := \Call{Cost}{\Vecx,\SetT_k}, k=\{i,g\}$ denote the cost-to-come of $\Vecx$ from $\Vecx\subinit$ in $\SetT_i$ or the cost-to-go from $\Vecx$ to $\Vecx\subgoal$ in $\SetT_g$ (infinite if outside of tree $\SetT_k$).  When the algorithm terminates, at least one sample $\Vecx\submeet \in \SetV_i \intersect \SetV_g$ exists where the two trees meet; we select the $\Vecx\submeet$ that minimizes the total cost-to-go such that $\CostFcn_n = \CostFcn_i(\Vecx\submeet) + \CostFcn_g(\Vecx\submeet)$.  We then bound the costs-to-come of each $\{\Vecy_m\}$ using an inductive hypothesis (Eq. 5) in Theorem VI.1 of \cite{ES-LJ-MP:15a}.  Finally, to bound $\CostFcn_n$ for \BFMT, we consider the two cases: (i) there exists some $\Vecy_m \in \SetV_i \intersect \SetV_g$, and (ii) there are no $\Vecy_m \in \SetV_i \intersect \SetV_g$.  Straightforward bounds can be derived in both cases to show that $\CostFcn_n \leq \CostFcn(\Vecy) + r_n$.  This then completes the proof.
	}
\end{proof}

\iftoggle{arXivVersion}{%
	\begin{remark}[Tightened bound for connection radius]
		As discussed in \cite{ES-LJ-MP:15a}, for the sake of clarity the constant term $4$ in the expression for $r_n$ is greater than is necessary for \cref{thm:pathtracing} to hold.  A more careful argument along the lines of the original \FMT AO proof \cite{LJ-MP:13} would suffice to show that
		a factor of $2$ satisfies the theorem as well.
	\end{remark}
}{%
}
	
\begin{remark}[Alternative termination criteria]
	\label{rem:stopping crit}
	The proof holds as well for the different expansion and termination criteria discussed in Section \ref{subsubsec: BFMT-Variations}. However, due to space constraints the details are omitted.
\end{remark}

We are now ready to show that \BFMT is asymptotically-optimal. The next theorem defines this formally.
\begin{theorem}[\BFMT asymptotic optimality]
	\label{thm:bfmt ao}
	Assume a $\delta$-robustly feasible path planning problem as defined in \cref{sec: Problem} with optimal path $\sigma\supopt$ of cost $\CostFcn\supopt$.  Then \BFMT converges \emph{in probability} to $\sigma\supopt$ as the number of samples $n \rightarrow \infty$.  Specifically, for any $\epsilon > 0$,
		\[ \liminfty[n] \Probability{ \CostFcn_n > (1+\epsilon) \CostFcn\supopt} = 0 \]
\end{theorem}

\iftoggle{arXivVersion}{%
	\begin{proof}
		The proof follows as a corollary to \cref{thm: BFMT Cost Comparison}.  By our $\delta$-robustly feasible assumption, we can find a strong $\delta$-clearance feasible path $\map{\sigma}{\unitinterval}{\Xfree}$ that approximates $\sigma\supopt$ with cost $\CostFcn(\sigma) < \left(1 + \sidefrac{\epsilon}{3}\right) \CostFcn\supopt$ (\textit{i.e.}, less than factor $\sidefrac{\epsilon}{3}$ from $\CostFcn\supopt$), for any $\epsilon > 0$.  By \cref{thm: BFMT Cost Comparison}, we can choose $n$ sufficiently large such that \BFMT returns an $\sidefrac{\epsilon}{3}$ cost approximation to the approximant:
		\begin{align*}
			\Probability{\CostFcn_n > \left(1 + \sidefrac{\epsilon}{3}\right)^2 \CostFcn\supopt} &< \Probability{\CostFcn_n > \left(1 + \sidefrac{\epsilon}{3}\right) \CostFcn(\sigma)} \\
				&= \BigO{n^{-\frac{\eta}{d}}\log^{-\frac{1}{d}} n}
		\end{align*}
		To approach the optimal path, let the number of samples $n \rightarrow \infty$.  It follows that, for any $\eta \geq 0$:
		\begin{align*}
			\liminfty[n] \Probability{ \CostFcn_n > \left(1 + \sidefrac{\epsilon}{3}\right)^2 \CostFcn\supopt} &< \liminfty[n] \BigO{n^{-\frac{\eta}{d}}\log^{-\frac{1}{d}} n} = 0 
		\end{align*}
		Now we relate this to the original claim.  First suppose that $\epsilon \leq 3$.  From $\left(1 + \sidefrac{\epsilon}{3}\right)^2 \leq 1 + \epsilon$, the event $\left\{\CostFcn_n > (1 + \epsilon) \CostFcn\supopt\right\}$ is a subset of the event $\left\{\CostFcn_n > \left(1+\sidefrac{\epsilon}{3}\right)^2 \CostFcn\supopt\right\}$, hence:
		\begin{align*}
			\liminfty[n] \Probability{ \CostFcn_n > (1+\epsilon) \CostFcn\supopt} &\leq \liminfty[n] \Probability{ \CostFcn_n > \left(1+\sidefrac{\epsilon}{3}\right)^2 \CostFcn\supopt } \!= \!0.
		\end{align*}
		Because the probability is monotone-decreasing in $\epsilon$ as $\epsilon$ increases, the statement holds for all $\epsilon > 3$ as well (to see this, apply \cref{thm: BFMT Cost Comparison} again for $m$ sufficiently large to handle $\epsilon = 3$; then by similar argument as above $\Probability{ \CostFcn_m > (1+\epsilon) \CostFcn\supopt} < \Probability{ \CostFcn_m > (1+3) \CostFcn\supopt} = \BigO{m^{-\frac{\eta}{d}}\log^{-\frac{1}{d}} m}$ and take the limit as $m \rightarrow \infty$).  Hence \mbox{$\liminfty[n] \Probability{ \CostFcn_n > (1+\epsilon) \CostFcn\supopt} = 0$} holds for arbitrary $\epsilon$, and we see that \BFMT converges in probability to the optimal path, as claimed.
	\end{proof}
}{%
	\begin{proof}
		We provide a brief sketch of the proof here but refer the interested reader to the arXiv version of the paper for full details.  The proof essentially follows as a corollary to \cref{thm: BFMT Cost Comparison}.  First, by our $\delta$-robustly feasible assumption, we know a strong $\delta$-clearance feasible path $\map{\sigma}{\unitinterval}{\Xfree}$ exists that closely approximates $\sigma\supopt$.  By \cref{thm: BFMT Cost Comparison}, we can choose $n$ sufficiently large such that \BFMT returns an $\sidefrac{\epsilon}{3}$ cost approximation $\CostFcn_n$ to the path approximant $\CostFcn(\sigma)$.  To approach the optimal path, let the number of samples $n \rightarrow \infty$.  We then show using an events bound and monotonicity arguments that the probability of returning a path more costly than a factor $\epsilon$ above $\CostFcn\supopt$ vanishes in the limit for arbitrary $\epsilon > 0$, and hence that \BFMT converges in probability to the optimal path, as claimed.
	\end{proof}
}

\iftoggle{arXivVersion}{%
\begin{remark}[Convergence rate]
	Note that we can also translate the convergence rate from \cref{thm: BFMT Cost Comparison} to the setup of \cref{thm:bfmt ao}, which does not require strong $\delta$-clearance.  For any $\epsilon > 0$, the optimal path can be approximated by a strong-$\delta$-clear path with cost less than $(1 + \epsilon) \CostFcn(\sigma)$ and we can focus on approximating that path to high-enough precision to still approximate the optimal path to within $(1 + \epsilon)$. Since the convergence rate in \cref{thm: BFMT Cost Comparison} only contains $\epsilon$ in the rate's constant, the big-O convergence rate remains the same.  This generalizes the convergence rate result in \cite{LJ-ES-AC-ea:15}, which only applied to a specific obstacle-free \state space, initial \state, and goal region.
\end{remark}
}{%
}

\subsection{Sampling and cost generalizations}
\label{subset:bfmt_disc}
It is worth mentioning that the asymptotic optimality (AO) properties of \BFMT are not limited to uniform sampling and arc-length cost functions.
\iftoggle{arXivVersion}{%
	For example, if one has prior information about areas that the optimal path is unlikely to pass through, it may be advantageous to consider a non-uniform sampling strategy that downsamples these regions. As long as the sampling density is lower-bounded by a \emph{positive} number over the \state space, \BFMT can be slightly altered (by merely increasing $r_n$ by a constant factor) to ensure it stays AO. The argument is analogous to that made in \cite{LJ-ES-AC-ea:15}, and essentially proceeds by making the search radius wide enough to balance out the detrimental effect of the lower sampling density (in some areas).  An additional common concern is when the cost is not arc-length, but some other metric or line integral cost.  In either case, \BFMT need only consider \emph{cost} balls instead of Euclidean balls when making connections.  Details on adjusting the algorithm and why the AO proof still holds can be derived from \cite{LJ-ES-AC-ea:15}. The argument basically shows that the triangle inequality either holds exactly (for metric costs) or approximately, and that this approximation goes away in the limit as $n \rightarrow \infty$.
}{%
	For non-uniform sampling, so long as sample density is lower-bounded by a \emph{positive} number over the \state space, one need only increase $r_n$ by a constant factor to ensure \BFMT stays AO.  Furthermore, to handle other metric or line integral costs, the Euclidean balls used when making node connections need only be replaced by \emph{cost} balls.  Details on adjusting the algorithm and why the AO proof still holds can be derived from \cite{LJ-ES-AC-ea:15}.
}

\section{Simulations}
\label{sec:sim}

In this section, we provide numerical path-planning experiments that compare the performance of \BFMTstar with other sampling-based, asymptotically-optimal planning algorithms (namely, \FMT, \RRTstar, and \PRMstar)\footnote{Existing state-of-the-art sampling-based, \bidirectional algorithms (namely, \RRTConnect and SBL) were initially also included.  However, average costs for \RRTConnect and SBL were roughly 2-4x greater, which occluded the details of other curves; they were thus omitted for clarity}.  Given a planning workspace and query, we aim to observe the quality of the solution returned as a function of the execution time allotted to the algorithm.  Here dynamic constraints are neglected and arc-length is used as path cost.  As a basis for quality comparison between incremental or "anytime" planners (such as \RRTstar) and non-incremental planners (such as \BFMTstar, which generate solutions via sample batches), we vary the number of samples drawn by the planners during the planning process (which in essence serves as a proxy to execution time).  Note \emph{sample count} has a different connotation depending on the planner that will not necessarily be the number of nodes stored in the constructed solution graph -- for \RRTstar (with one sample drawn per iteration), this is the number of iterations, while for \FMTstar, \PRMstar, and \BFMTstar, this is the number of free space samples taken during initialization.

\subsection{Simulation Setup}
To generate simulation data for a given experiment, we queried the planning algorithms once each for a series of sample counts, recorded the cost of the solution returned, the planner execution time\footnote{Code for all experiments was written in C++.  Corresponding programs were compiled and run on a Linux-operated PC, clocked at 2.4 GHz and equipped with 7.5 GB of RAM.}, and whether the planner succeeded or not, then repeated this process over 50 trials.  To ensure a fair comparison, each planning algorithm was tested using the Open Motion Planning Library (OMPL) v1.0.0 \cite{IAS-MM-LEK:12}, which provides high-quality implementations of many state-of-the-art planners and a common framework for executing motion plans.  In this way, we could ensure that all algorithms employed the \emph{exact same} primitive routines (\textit{e.g.}, nearest-neighbor search, collision-checking, data handling, \textit{etc}), and measure their performances fairly.  Regarding implementation, \BFMT, \FMT, and \PRMstar used $\eta = 0$ from \cref{thm: BFMT Cost Comparison} for the nearest-neighbor radius $r_n$ in order to satisfy the theoretical bounds provided in \cref{sec: BFMT* Proofs} and \cite{SK-EF:11}.  For \RRTstar, we used the default OMPL settings; namely, a 5\% goal bias and a steering parameter equal to 20\% of the maximum extent of the \state space (except for the $\alpha$-puzzle, in which case a value of 1.1 was found to work much better). For \FMT, we included the same \Call{\InsertNewSample}{} routine as \BFMT for \state resampling upon failure.  For all algorithms, early termination (\textit{e.g.}, using \Call{Terminate}{} for \BFMT) was suppressed by defining a 1000 second time limit, well above each planner's worst-case execution time.

Before proceeding, note that each marker shown on the plots throughout this section represents a single simulation at a fixed sample count.  The points on the curves, however, represent the mean cost/time of \emph{successful} algorithm runs \emph{only} for a particular sample count, with error bars corresponding to one standard deviation of the 50 run sample mean.\footnote{Standard deviation of the mean indicates where we expect with one-$\sigma$ confidence the distribution mean to lie based on the 50-run sample mean, and is related to the standard deviation of the distribution by $\sigma_{\mu} = \sigma/\sqrt{50}$.}
Sample counts varied from the order of 200 to 2000 points for 2D problems, from 1000 to 30000 points for 3D problems, and 500 to 4000 points for the hypercube examples.

\subsection{Results and Discussion}
Here we present benchmarking results (average solution cost versus average execution times and success rates) comparing \BFMT to other state-of-the-art sampling-based planners.  Three benchmarking test scenarios were considered: 
(1) a 2D ``bug trap'' and (2) a 2D ``maze'' problem for a convex polyhedral robot in the $\SE(2)$ \state space, as well as (3) a challenging 3D problem called the ``$\alpha$-puzzle'' in which we seek to untangle two loops of metal (non-convex) in the $\SE(3)$ \state space.  All problems were drawn directly from OMPL's bank of tests, and are illustrated in \cref{fig: scenarios}.  In each case, collision-checks relied on OMPL's built-in collision-checking library, FCL.  Additionally, to tease out the performance of \BFMT relative to \FMT in high-dimensional environments, we also studied a point mass robot moving in cluttered unit hypercubes of 5 and 10 dimensions.\footnote{We populated the space to 50\% obstacle coverage with randomly-sized, axis-oriented hyperrectangles.  $\Vecx\subinit$ was set to the center at $\left[0.5, \ldots, 0.5\right]$, with the goal $\Vecx\subgoal$ at the ones-vector (\textit{i.e.}, $\left[1, \ldots, 1\right]$).}

\begingroup
\newcommand{\Figures}{%
	2D_Bug_Trap/$\SE(2)$ bug trap,%
	2D_Maze/$\SE(2)$ maze,%
	3D_Alpha/\mbox{$\SE(3)$ $\alpha$-puzzle}%
}
\begin{figure}[ht!]
	\foreach \figname/ \figcaption [count=\ni] in \Figures {
		\begin{subfigure}[b]{0.15\textwidth}
			\centering\includegraphics[width=\textwidth]{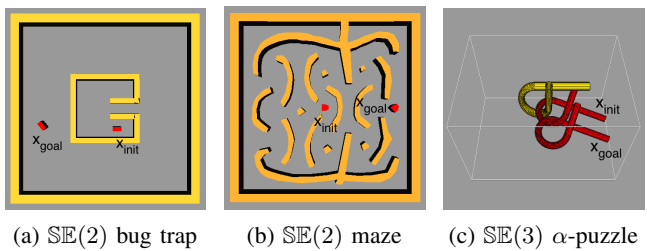}
			\caption{\figcaption}
			\label{fig: Performance_\figname}
		\end{subfigure}
	}
	\caption{Depictions of the OMPL rigid-body planning problems}
	\vspace{-0.5em}
	\label{fig: scenarios}
\end{figure}
\endgroup

\Cref{fig: Performance_OMPL_Problems} shows the results for each \BFMT, \FMT, \RRTstar, and \PRMstar.  Performance here is measured by execution time on the x-axis and solution cost on the y-axis---high quality data points are therefore located in the lower-left corner (low-cost solutions obtained quickly).  The plots reveal that both \FMT and \BFMT for the most part outperform \RRTstar  as well as \PRMstar.  In particular, \BFMT and \FMT achieve higher success rates (always a flat 100\% for the cases studied) in shorter time.  To extract further information, we need to examine each test in detail.

In the Bug Trap and Maze problems, \BFMT notably generates the same cost-time curve as \FMT (meaning they return solutions of very similar cost for a given sample count), but with data points shifted to the left (indicating they were obtained in shorter execution time).  Though not shown due to slow running times for \PRMstar (whose results had to be truncated to clarify detail), all planners appear to tend towards similar low-cost solutions as more execution time was allocated.  However \BFMT and \FMT seem to converge to an optimum much faster, particularly for the Maze problem (on the order of 1.5 and 2.0 seconds respectively, compared to 3-4 seconds for \RRTstar and 5-7 seconds for \PRMstar).  This contrast becomes even more evident for the $\alpha$-puzzle.  Here we see an unusual spread of solutions -- one in a band at around 500 cost and another at around 275.  These indicate the presence of two solution types, or \emph{homotopy classes}: one corresponding to the true $\alpha$-puzzle solution, and another less-efficient path.  This appears to have yielded a ``bump'' in the \BFMT cost-curve, where increasing the sample count momentarily gives an increased average cost.  We believe this is a result of how \BFMT trees interconnect; at this count, by unlucky circumstance, the longer homotopy seems to be found first more often than usual.  But as proved in \cref{sec: BFMT* Proofs}, the behavior disappears as $n \rightarrow \infty$.  Note \RRTstar seems to avoid this issue through goal biasing.  Despite the difficult problem structure, \BFMT finds the cheaper homotopy faster than other planners, with many more of its data points clustered in the lower-left corner, generally at lower costs and times than \RRTstar and of equal quality but faster times than \FMT.

These results suggest that \BFMT tends to an optimal cost at least as fast as the other planners, and sometimes much faster.  To shed light on the relative performance of \FMT and \BFMT further, we compare them in higher dimensions.  Results for the 5D and 10D hypercube are shown in \cref{fig: Performance_HypercubicalSpace} (success rates were again at 100\%, and were thus omitted).  Here \BFMT substantially outperforms \FMT, particularly as dimension increases, with convergence in roughly 0.5 and 1.4 seconds (5D), and 5 and 20 seconds (10D) on average.  
This suggests that \emph{reachable volumes} play a significant role in their execution time.  The relatively small volume of reachable \states around the goal at the corner implies that the reverse tree of \BFMT expands its wavefront through many fewer states than the forward tree of \FMT (which in fact needlessly expands towards the zero-vector); tree interconnection in the \bidirectional case prevents its forward tree from growing too large compared to unidirectional search.  This is pronounced exponentially as the dimension increases.  In trap or maze-like scenarios, however, \bidirectionality does not seem to change significantly the number of states explored by the marching trees, leading to comparable performance for the $\SE(2)$ bug-trap and maze.
Note we expect a greater contrast in execution times in favor of \BFMT as the cost of collision-checking increases, such as with many non-convex obstacles or in time-varying environments.


\begingroup
\newcommand{\Problems}{%
	2D_Bug_Trap/a bug trap in $\SE(2)$-space,%
	2D_Maze/a maze in $\SE(2)$-space,%
	3D_Alpha/the $\SE(3)$ "alpha" puzzle%
}
\newcommand{\HypercubeProblem}{HR}
\renewcommand{\ObstacleCoverages}{50}
\begin{figure}[ht!]
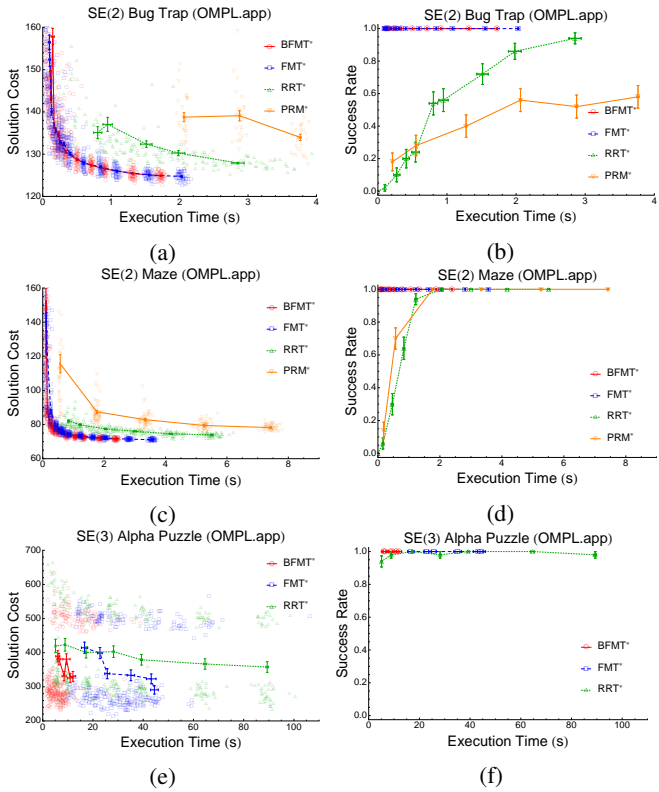

	\foreach \problem/ \probname in \Problems {
			\centering
			\begin{subfigure}[c]{\figwidth}
				\includegraphics[width=\textwidth]{\problem_\CostVsTime.pdf}
				\caption{}
			\end{subfigure}
			\hfill
			\begin{subfigure}[c]{\figwidth}
				\includegraphics[width=\textwidth]{\problem_\SuccessVsTime.pdf}
				\caption{}
			\end{subfigure}
	}
	\caption{Simulation results for the three OMPL scenarios.}
	\vspace{-0.5em}
	\label{fig: Performance_OMPL_Problems}
\end{figure}
\foreach \coverage in \ObstacleCoverages {
	\begin{figure}[ht!]
		\centering
		\foreach \dimension in \Dimensions {
			\begin{subfigure}[c]{\figwidth}
				\includegraphics[width=\textwidth]{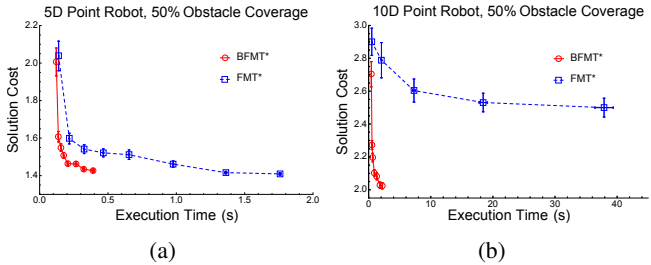}
				\caption{}
			\end{subfigure}
		}
		\caption{\FMT and \BFMT results for 5D and 10D cluttered hypercubes (\coverage\% coverage; all success rates were $100\%$).}
		\vspace{-1em}
		\label{fig: Performance_HypercubicalSpace}
	\end{figure}
}
\endgroup

\section{Conclusion}
\label{sec:conc}

In this paper, we presented a \bidirectional, sampling-based, asymptotically-optimal motion planning algorithm named \BFMT, for which we rigorously proved its optimality and characterized its convergence rate -- arguably firsts in the field of \bidirectional sampling-based planning.  Numerical experiments in $\reals^d$, $\SE(2)$, and $\SE(3)$ revealed that \BFMT tends to an optimal solution at least as fast as its state-of-the-art counterparts, and in some cases significantly faster.  Convergence rates are expected to improve with parallelization, in which each tree is grown using a separate CPU.

Future research will examine \BFMT's interaction with more advanced techniques, such as adaptive sampling near narrow passages or sample biasing in \Call{\InsertNewSample}{} (\cref{alg: InsertNewSample}) towards failed wavefronts.  We also plan to extend \BFMT to dynamic environments through lazy re-evaluation (leveraging its tree-like forward and reverse path structures) in a way that reuses previous results as much as possible.  Maintaining bounds on run-time performance and solution quality in this new context will be the greatest challenges.  Ultimately, we hope that \BFMT will enable fast, easy-to-implement planning and re-planning in a wide range of time-varying scenarios, much as we have shown here for the static case.

\bibliographystyle{IEEEtran}
\bibliography{\bibfiles}

\end{document}